\def\year{2018}\relax
\documentclass[letterpaper]{article} 
\usepackage{aaai18}  
\usepackage{times}  
\usepackage{helvet}  
\usepackage{courier}  
\usepackage{url}  
\usepackage{graphicx}  
\usepackage{enumerate}
\usepackage{colortbl}
\usepackage{subfigure}
\usepackage[fleqn]{amsmath}
\usepackage{amsthm}

\newtheorem{myTheo}{Theorem}

\theoremstyle{definition}
\newtheorem{myExm}{Example}
\newtheorem{myDef}{Definition}

\definecolor{mygray}{gray}{.2}
\definecolor{mypink}{rgb}{.99,.91,.95}
\definecolor{mycyan}{cmyk}{.3,0,0,0}

\begin{document}
%
\title{A Logical Model for Supporting Social Commonsense Knowledge Acquisition} 
\author{Zhenzhen Gu \ and  Cungen Cao \ and Ya Wang \ and Yuefei Sui\\
	   \{\textit{guzhenzhen, cgcao, wangya, yfsui}\}@ict.ac.cn\\
	Institute of Computing Technology, Chinese Academy of Sciences, Beijing, China \\
}

\maketitle
\begin{abstract}  
To make machine exhibit human-like abilities in the domains like robotics and conversation, social commonsense knowledge (SCK), i.e., common sense about social contexts and social roles, is absolutely necessarily. Therefor, our ultimate goal is to acquire large-scale SCK to support much more intelligent applications. Before that, we need to know clearly what is SCK and how to represent it, since automatic information processing requires data and knowledge are organized in structured and semantically related ways. For this reason, in this paper, we identify and formalize three basic types of SCK based on first-order theory. Firstly, we identify and formalize the interrelationships, such as \textit{having-role} and \textit{having-social\_relation}, among social contexts, roles and players from the perspective of considering both contexts and roles as first-order citizens and not generating role instances. Secondly, we provide a four level structure to identify and formalize the intrinsic information, such as \textit{events} and \textit{desires}, of social contexts, roles and players, and illustrate the way of harvesting the intrinsic information of social contexts and roles from the exhibition of players in concrete contexts. And thirdly, enlightened by some observations of actual contexts, we further introduce and formalize the embedding of social contexts, and depict the way of excavating the intrinsic information of social contexts and roles from the embedded smaller and simpler contexts. The results of this paper lay the foundation not only for formalizing much more complex SCK but also for acquiring these three basic types of SCK.

\end{abstract} 


\section{Introduction}  
To make machine achieve human-level intelligence in the domains such as robotics, conversation, natural language understanding and so on, without replicating human cognition directly, basic commonsense knowledge of our society and people will be necessary, since there exist a lot of problems that are easy for people to understand and answer but can be difficult for artificial systems without our common sense (Davis and Marcus 2015, Davis 2017). Our society mainly consists of diverse social contexts, such as \textit{school} and \textit{hospital}. And people interpret their everyday even whole life by entering into different social contexts to play different social roles, such as \textit{daughter}, \textit{student} and \textit{doctor}. 
Therefor, our ultimate goal is to acquire large-scale social commonsense knowledge (SCK), i.e., common sense about social contexts and social roles, to support much more intelligent applications. 
Before that, we need to know clearly what is SCK and how to represent it, i.e., what to acquire and how to formalize the acquired knowledge, 
since automatic and intelligent information processing requires knowledge and data to be organized in structured and semantically related ways.  
For this reason, in this paper, we identify and formalize three basic types of SCK to lay the foundation for formalizing more complex SCK and acquiring these three basic types of SCK.

Roles are very important both theoretically and practically in modeling the real world phenomena around us. And many theories of roles have been proposed for different application areas and different practical problems (Steimann 2000; Masolo et al. 2004; Loebe 2007; Mizoguchi et al. 2015; Bera, Burton-Jones and Wand 2017). For example, Masolo et al. (2004) and Boella and van der Torre (2006) discuss and formalize the distinguish features of social roles and organizations, and Boella and van der Torre (2007) study the ontological properties of social roles in multi-agent systems. For these works, roles are modeled either as types (classes, universals or unary predicts), or relations (multivariate predicts), or objects (instances). And for the works modeling roles as types (classes, universals and unary predicts), role instances are modeled either as the players or as adjunct entities that existentially depend on their players but are disjoint from them. 
And up to now, no unified and common agreed formalization and representation of roles exist (Mizoguchi et al. 2015, Burton-Jones and Wand 2017, Genovese 2007, Loebe 2007, Boella, van der Torre and Verhagen 2006), since different authors usually emphasize different aspects of roles and it is difficult to piece together these models on roles.  

Based on the current works on roles, we first identify and formalize the interrelationships among social contexts, roles and players from the following two distinctive perspectives. Firstly, we model both social contexts and roles, no matter abstract or concrete, as first-order citizens, i.e., instances. By this way, we can not only successfully avoid the controversy about the dimensions of social roles but also describe the meta-knowledge of social contexts and roles without exceeding the first-order logic category. 
And secondly, we do not generate social role instances, with the motivation of simplifying depiction and reducing the number of instances significantly. Normally, a player does not play the same role multiple times in the same context and same period. 
Therefor, a player attached with a social role can exactly denote a role instance in a context. For example, the description ``Bob is a student in university u during time interval t" can be captured as $ \textit{play}(\textit{Bob}, \textit{student}, u, t) $ by our theory. Comparing with the representation of first generating a role instance $ r $ of the role \textit{student} then declaring \textit{Bob} playing \textit{r} and \textit{u} having the role \textit{r}, our representation is more succinct and intuitive. Moreover, besides the widely accepted \textit{having-role} and \textit{playing} interrelations, other interrelations such as having and playing social relations, are also identified and formalized.   

For the context \textit{school}, besides the social roles, such as \textit{teacher} and \textit{student}, linked into it, we can associate a lot of things such as \textit{enrolling} and \textit{teaching}. And for the context \textit{hospital},  the things like \textit{treating} and  \textit{operation} can be associated. This means social contexts have their intrinsic information, and so for social roles and players. For this reason, we further propose a four level structure to identify and formalize the intrinsic information, such as \textit{events} and \textit{desires}, of social contexts, roles and players from different angles and with different granularity. Besides, the way of harvesting the intrinsic information of social contexts, roles and social relations from the performance of the corresponding players in concrete contexts is also illustrated. 

It is undeniable that the social contexts, such as schools, hospitals and companies, we participated in  are very complex, since, they usually contains a lot of social roles, each role may have a lot of players, and different social roles usually behave totally differently. 
Such complexity makes acquiring SCK, a hard challenge problem.  By deeply analyzing and observing some actual social contexts, we found that like building block, a social context, such as \textit{school}, can be spliced up by some much smaller and simpler contexts, such as \textit{classroom context}. 
And a player in a social context by entering into some smaller and simpler contexts, such as \textit{classroom context} to interpret the meaning and mission of the role, such as \textit{teacher}, it played. 
This enlightens us to introduce the concept of social contexts embedding to formalize such phenomenon of social contexts and roles. 

Social context embedding can simplify the task of CSK acquisition from the following three aspects. Firstly, even a social context has a very long life-cycle (we can not observe it from its start to its finish), the number of smaller contexts embedded into it is usually limited. 
Secondly, even ``uncorrelated" social contexts may have the same types of embedded contexts, such as \textit{conference} and \textit{conversation}. This means, after acquired the common sense of the embedded contexts in one context, for other contexts, the same typed embedded contexts can be ignored. Finally, from the role perspective, a social role usually enters into some not all the embedded contexts of a social context. For example, a \textit{auditor} usually just enters into the classroom contexts. By just considering the related embedded contexts, the task of acquiring common sense of social roles can be simplified.

\paragraph{Contributions.} The main contribution of this paper can be summarized as follows:  
\begin{enumerate}
	\item The interrelationships, such as \textit{having-role} and \textit{having-social\_relations}, among social contexts, roles and players are identified and formalized from the perspective of considering contexts and roles as first order citizens and not generating role instances.
	\item A four level structure is provided to identify and formalize the intrinsic information, such as \textit{events} and \textit{desires}, of social contexts, roles and players, and the way of harvesting the intrinsic information of social contexts, roles and social relations from the exhibition of players in instance contexts is also illustrated;  
	\item The embedding of social contexts is introduced and formalized and the way of excavating the intrinsic information of social contexts, roles and social relations from the embedded smaller and simpler contexts is depicted. 
\end{enumerate}   


\section{Interrelationships among Social Contexts, Social Roles and Players}     
 
As illustrated in Fig.\ref{Fig:Relation}, the widely accepted interrelationships among social contexts, roles and players are social contexts having social roles and players playing social roles in social contexts. For example, a hospital has social roles such as \textit{doctor} and \textit{patient}, and each \textit{doctor} and each \textit{patient} are played by human beings. In this section, we formalize the interrelationships among social contexts, roles and players, containing not only \textit{having} and \textit{playing} social roles but also \textit{having} and \textit{playing} social relations generated by social roles and the \textit{containment} relation among social contexts (roles) as well as the relation of roles and social relations in abstract contexts and in instance contexts. 
The primitive predicts used and their meanings as well as the domain restrictions of the arguments of relational predicts are shown in Fig. \ref{Fig:Predict1}. 

\begin{figure}[!htbp] 
\centering
\small
\includegraphics[scale=0.45]{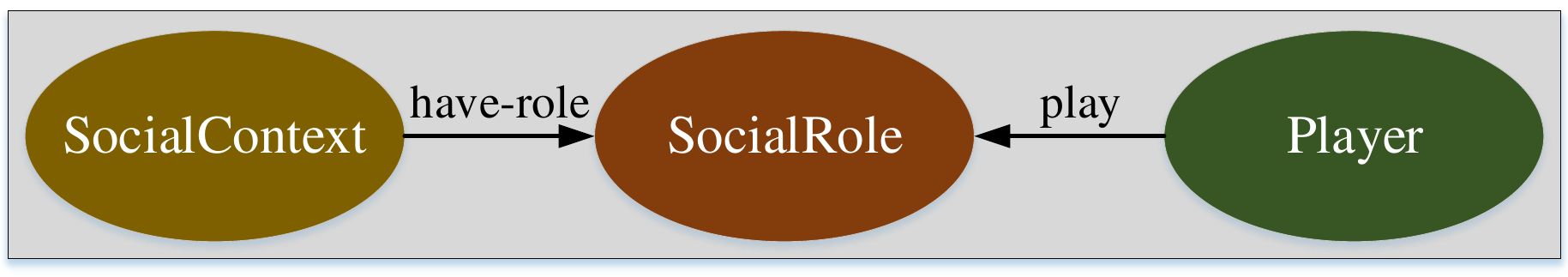}
\caption{The playing and having-role relationships among social contexts, roles and players.}
\label{Fig:Relation}
\end{figure}

\begin{figure}[!htbp]
\centering
\small
\subfigure[]{
\begin{tabular}{|l|}  
\hline
\!$ \bullet\  \textit{SC}(C) $, $ \textit{SR}(r) $:\ \ {\color{mygray}\textit{$ C $ is an abstract context, $ r $ is a social role }};\\
\!$ \bullet\   \textit{IC}(c) $, $ \textit{PL}(p) $:\ \ {\color{mygray}\textit{$ c $ is an instance context, \textit{p} is a player}};\\
\!$ \bullet\   \textit{TI}(t) $:\ \ {\color{mygray}\textit{$ t$ is a time interval}};\\
\!$ \bullet\  t_{1}\!\prec\! t_{2} $:\ \ {\color{mygray}\textit{time interval $ t_{1} $ is contained by time interval $ t_{2} $}}; \\
\!$ \bullet\  \textit{insC}(c,C) $:\ \ {\color{mygray}\textit{$ c $ is an instance of the abstract context $ C $}}; \\
\!$ \bullet\   \textit{hasR}(c,r) $:\ \  {\color{mygray}\textit{context $ c $ has role $ r $}}; \\
\!$ \bullet\  \textit{hasCoR}(c,r_{1},r_{2}) $:\ \  {\color{mygray}\textit{context $ c $ has the social relation $ (r_{1}, r_{2}) $ }}\\
\ \  {\color{mygray}\textit{generated by role $ r_{1} $ and $ r_{2} $}}; \\
\!$ \bullet\  \textit{play}(p,r,c,t) $:\ \  {\color{mygray}\textit{player $ p $ plays role $ r $ in context $ c $ during $ t $}}; \\
\!$ \bullet\  \textit{coPlay}(p_{1}, p_{2}, r_{1}, r_{2},c,t) $:\ \  {\color{mygray}\textit{$ p_{1} $ and $ p_{2} $ materialize the }} \!\!\!\\
\ \  {\color{mygray}\textit{social relation $ (r_{1}, r_{2}) $ in $ c $}}\\
\hline
\end{tabular}
}
\subfigure[]{
\begin{tabular}{|lrcl|}
\hline
\!(1)&$ \textit{insC}(c,C)$ \!\!\!\!\!\!\!\! &$\to$ &\!\!\!\!\!\!\!\! $ \textit{IC}(c)\wedge \textit{SC}(C) $\\
\!(2)&$ \textit{hasR}(c,r)$ \!\!\!\!\!\!\!\! &$\to$ &\!\!\!\!\!\!\!\! $ (\textit{SC}(c)\vee\textit{IC}(c))\wedge \textit{SR}(r) $\\
\!(3)&\!\!\!\!\!\!\!\!\!\!$ \textit{hasCoR}(\vec{x})$ \!\!\!\!\!\!\!\! &$\to$ &\!\!\!\!\!\!\!\! $ \textit{hasR}(\vec{x}^{1},\vec{x}^{2})\wedge \textit{hasR}(\vec{x}^{1},\vec{x}^{3}) $\\
\!(4)&$ \textit{play}(\vec{x})$ \!\!\!\!\!\!\!\! &$\to$ &\!\!\!\!\!\!\!\! $ \textit{PL}(\vec{x}^{1})\wedge\!\textit{SR}(\vec{x}^{2})\wedge\!\textit{ISC}(\vec{x}^{3})\wedge\!\textit{TI}(\vec{x}^{4})$\\
\!(5)&$ \textit{coplay}(\vec{x})$ \!\!\!\!\!\!\!\! &$\to$ &\!\!\!\!\!\!\!\! $ \textit{play}(\vec{x}^{1},\vec{x}^{3},\vec{x}^{5},\vec{x}^{6})\wedge \textit{play}(\vec{x}^{2},\vec{x}^{4},\vec{x}^{5},\vec{x}^{6})$\!\!\!\\
\hline
\end{tabular}
}
\caption{Primitive predicts used and their means (a) as well as the domain restrictions of the arguments of relational predicates (b), where $ \vec{x}_{i} $ denotes the $ i$-\textit{th} element of the vector $ \vec{x} $.}
\label{Fig:Predict1}
\end{figure}

In Fig.\ref{Fig:Predict1}, the predicts $ \textit{hasCoR} $ and $ \textit{coPlay} $ respectively denote having and playing social relations. For example, $ \textit{hasCoR}(\textit{school}, \textit{teacher}, \textit{student}) $ denotes the context \textit{school} has the social relation generated by \textit{teacher} and \textit{student}, i.e., the \textit{teacher relationship}. And $ \textit{coPlay}(\textit{Lucy}, \textit{Bob}, \textit{teacher}, \textit{student}, c) $ denotes \textit{Lucy} and \textit{Bob} materialize (play) the \textit{teacher relationship} in \textit{c}, i.e., in context \textit{c}, \textit{Lucy} is \textit{Bob}'s teacher and \textit{Bob} is \textit{Lucy}'s student. Moreover, axiom 3 depicts that if a context has a social relation generated by two roles then this context also has these two roles, and if two players materialize a social relation in a context then they respectively play a corresponding role in this context. These two restrictions accord with our common sense. For example, if a context has the social relation \textit{teacher\_relationship}, then we naturally consider that this context has the roles \textit{teacher} and \textit{student}. And if someone tells us that \textit{Lucy} and \textit{Bob} are \textit{teacher\_relationship} in a university, then we also know that \textit{Lucy} and \textit{Bob} are respectively \textit{teacher} and \textit{student} of this university.

\subsection{The \textit{having-role}, \textit{having-social\_relation} and \textit{playing} relationships}  
The \textit{having-role} relationship between social contexts and roles are illustrated in the following two axioms. Concretely, axiom 6 depicts that each social context, no matter abstract or instance, has some social roles. And axiom 7 illustrates that each social role is linked into some social context. 
\[
\!\!\!\!\!\!\!\!\!
\begin{array}{lrcl}
(6)& \textit{SC}(c)\vee \textit{IC}(c)\!\!\!\!\!&\leftrightarrow&\!\!\!\! \exists r.\textit{hasR}(c,r)\\
(7)& \textit{SR}(r)\!\!\!\!\!&\leftrightarrow&\!\!\!\! \exists c.\textit{hasR}(c,r)\\
\end{array}
\]
Different to the existing works on roles, in our model, the roles linked into instance contexts are role types rather than role instances. This is because we do not model role instances in our theory. More importantly, we want to capture the phenomenon that even if an instance context does not have any instances of a role in some period (job vacancy), it still has this role. 

The roles in social contexts are not isolated. They generate social relations in social contexts. In other world, social contexts have social relations: 
\[
\!\!\!\!\!\!\!\!\!
\begin{array}{ll}
(8)& \textit{SC}(c)\vee \textit{IC}(c)\to\exists r_{1},r_{2}.\textit{hasCoR}(c,r_{1},r_{2})\\
\end{array}
\] 
For example, families have the social relationship of \textit{father\_and\_son} and schools have the social relationship of \textit{teacher\_and\_student}.  

And the roles and social relations linked into abstract contexts consist of the roles and social relations linked into their instances, i.e.,  instance contexts can just have roles and social relations linked into their abstract contexts and each role and each social relation of an abstract context will be inherited by one of its instance context:
 \[
\!\!\!\!
\begin{array}{ll}
(9) & \textit{SC}(C)\wedge \textit{hasR}(C,r)\leftrightarrow\exists c.\textit{insC}(c,C)\wedge \textit{hasR}(c,r)\\
(10) & \textit{SC}(C)\wedge \textit{hasCoR}(C,r_{1},r_{2})\\
     & \leftrightarrow \exists c.\textit{insC}(c, C)\wedge \textit{hasCoR}(c,r_{1},r_{2})\\
\end{array}
\]
Note that in axiom 9-10, ``$ \exists c $" cannot be replaced by ``$ \forall c $", since even for the contexts of the same types may contain different social roles. For example, some hospitals have the role of \textit{orthopedist}, while some do not. Based on axiom 9-10, the social roles and social relations of abstract contexts, i.e., context types, can be harvested from their instances.

Next, we illustrate the \textit{playing} relationship among players, social roles and contexts.

The next axiom depicts that each player will play a social role and participate in a materialization of a social relation in an instance social context:
\[
\!\!\!\!\!\!\!\!\!
\begin{array}{llll}
(11) & \textit{PL}(p)\!\!\!\!& \leftrightarrow &\!\!\!\!\exists r,c,t.\textit{play}(p,r,c,t)\\
&&&\!\!\!\wedge \exists r,p_{1},r_{1},c,t.\textit{coPlay}(p,p_{1},r,r_{1},c,t)\\
\end{array}
\]
And for an instance context, players can only play the roles and social relations linked into this context and each role and each social relation of this context will be played by some players: 
\[
\!\!\!\!\!\!\!\!\!
\begin{array}{ll}
(12)&\textit{IC}(c)\wedge \textit{hasR}(c,r)\leftrightarrow \exists p, t.\textit{play}(p,r,c,t)\\
(13)& \textit{IC}(c)\wedge \textit{hasCoR}(c,r_{1},r_{2})\\
    & \leftrightarrow \exists p_{1},p_{2},t.\textit{coPlay}(p_{1},p_{2},r_{1},r_{2},c,t)\\
\end{array}
\]
Then by axiom 12-13, the social roles and social relations linked into instances contexts can be excavated from the roles and social relations exhibited by the players in these contexts.  

Moreover, as illustrated in the axioms below, the \textit{playing} and \textit{co-playing } relationships are enduring: 
\[
\!\!\!\!\!\!\!\!\!
\begin{array}{crcl}
(14)&\textit{play}(\vec{x},t)\!\!\!\!&\to&\!\!\!\!\forall t_{1}.(t_{1}\!\prec t\to \textit{play}(\vec{x},t_{1}))\\
(15)&\textit{coPlay}(\vec{x},t)\!\!\!\!&\to&\!\!\!\!\forall t_{1}.(t_{1}\!\prec t\to \textit{coPlay}(\vec{x},t_{1}))\\
\end{array}
\]
For example, if \textit{Bob} is a student of a school $ s $ during time interval $ t $, i.e., $ \textit{play}(\textit{Bob}, \textit{student}, s, t) $, then for each time interval $ t_{1} $ contained  by $ t $, \textit{Bob} is also a student of $ s $ during $ t_{1} $, i.e., $ \textit{play}(\textit{Bob}, \textit{student}, s, t_{1}) $ holds.

\subsection{The \textit{containment} relationship} 
Besides \textit{having-role}, \textit{having-social\_relation} and \textit{playing}, there universally exist \textit{containment} relation among abstract social contexts and among social roles. For example, it is generally acknowledged that each middle school is also a school and each middle school student is also a student. Thus, the context \textit{middle\_school} and role \textit{middle\_school\_student} are respectively sub-concepts of the context \textit{school} and role \textit{student}.

In our model, both abstract contexts and instance contexts are captured. Thus, the \textit{containment} relation among abstract social contexts can easily be captured by the containment relationship of their instances.

\begin{myDef} 
We use $ \textit{isAC}(C_{1},C_{2}) $ to denote that abstract context $ C_{1} $ is a sub-context of abstract context $ C_{2} $, and it is formally defined as:
\[
\!\!\!
\begin{array}{clll}
(d1) & \textit{isAC}(C_{1},C_{2}) \!\!\!\!&\stackrel{\textit{def}}{=}&\!\!\!\!\textit{SC}(C_{1})\wedge \textit{SC}(C_{2})\\
&&&\!\!\!\!\wedge(\forall c.\textit{insC}(c, C_{1})\to \textit{insC}(c, C_{2}))\\
\end{array}
\]
\label{def:isAC}
\end{myDef}

As mentioned earlier, for simplicity and clarity, we do not generate social role instances. Instead, concrete roles of instance contexts are denoted by players attached with social roles. Therefor, as illustrated in the definition below, the \textit{containment} relation among abstract social roles can be depicted from the player perspective.  

\begin{myDef}
We use $ \textit{isAR}(r_{1},r_{2}) $ to denote role $ r_{1} $ is a sub-role of role $ r_{2} $, and it is formally defined as:
\[
\!\!\!\!\!\!\!\!\!\!
\begin{array}{clll}
(d2)& \textit{isAR}(r_{1},r_{2}) \!\!\!\!& \stackrel{\textit{def}}{=} &\!\!\!\! \textit{SR}(r_{1})\wedge \textit{SR}(r_{2}) \\
&&&\!\!\!\!\wedge(\forall p,c,t.\textit{play}(p,r_{1},c,t) \\
&&&\qquad\quad\ \ \to\textit{play}(p,r_{2},c,t))\\
\end{array}
\]
\label{def:isAR}	
\end{myDef} 

Based on Definition \ref{def:isAC} and \ref{def:isAR}, we can further obtain the reflexivity and transitivity of the \textit{sub-context} and \textit{sub-role} relationships as well as the extensions of \textit{having-role} on the \textit{sub-role} and \textit{sub-context} relations , i.e., if a context $ c $ has role $ r $ then for each sup-role $ r' $ of $ r $, $ c $ also has the role $ r' $, and for each sup-context $ c' $ of $ c $, $ c' $ also  has the role $ r $.

\begin{myTheo}
For contexts and roles, axiom $ c1 $-$ c6 $ hold: 
\[
\!\!\!\!\!\!\!\!
\begin{array}{crcl}
(c1)&\textit{SC}(C)\!\!\!\!&\to&\!\!\!\! \textit{isAC}(C,C) \\
(c2)&\textit{SR}(r)\!\!\!\!&\to&\!\!\!\!\textit{isAR}(r,r)\\
(c3)&\textit{isAC}(C_{1},C_{2})\wedge\textit{isAC}(C_{2},C_{3})\!\!\!\!&\to&\!\!\!\! \textit{isAC}(C_{1},C_{3})\\
(c4)&\textit{isAR}(r_{1},r_{2})\wedge\textit{isAR}(r_{2},r_{3})\!\!\!\!&\to&\!\!\!\! \textit{isAR}(r_{1},r_{3})\\
(c5)&\textit{hasR}(c,r_{1})\wedge\textit{isAR}(r_{1},r_{2})\!\!\!\!&\to&\!\!\!\! \textit{hasR}(c,r_{2})\\
(c6)&\textit{hasR}(C_{1},r)\wedge \textit{isAC}(C_{1},C_{2})\!\!\!\!&\to&\!\!\!\! \textit{hasR}(C_{2},r)\\
\end{array}
\]
\label{The:IsA}
\end{myTheo} 
\begin{proof}
(Sketc) By axiom $ d1 $-$ d2 $, axiom $ c1$-$c4 $ hold trivially. By axiom 9, 12 and $ d2 $, axiom $ c5 $ holds. And by axiom 9 and $ d1 $, axiom $ c6 $ holds. 
\end{proof}

Note that, the extensions of \textit{having-role} on the \textit{sub-role} and \textit{sub-context} relations accord with the common sense in the real world. For example, \textit{teacher} is a sub-role of \textit{staff}, thus if a school has the role \textit{teacher} then we naturally consider that this school also has the role \textit{staff}. On the other hand, no matter the \textit{professor} role in \textit{university} or the \textit{middle\_school\_teacher} in \textit{middle\_school}, they are both roles in the \textit{school} category. 

\subsection{Example and some comment}     

\begin{myExm}
The description "\textit{Lucy is a teacher of student Bob in university u during t. Meanwhile, Lucy is also a doctor of hospital h. Moreover, Lucy is a doctor of hospital h' during t'. Furthermore, teacher is a sub-role of staff.}", can be formalized by our model as the following statements: 
\[
\!\!\!\!\!
\begin{array}{|ll|}
\hline
\!(e1) &\textit{coPlay}(\textit{Lucy}, \textit{Bob}, \textit{teacher}, \textit{student}, u, t)\\
\!(e2)&\textit{play}(\textit{Lucy},\textit{doctor}, h, t),\ \textit{play}(\textit{Lucy},\textit{doctor}, h', t') \\
\!(e3)&\textit{insC}(u, \textit{University}), \textit{insC}(h, \textit{Hospital})\\
\!(e4)& \textit{insC}(h', \textit{Hospital}), \textit{isAR}(\textit{teacher},\textit{staff}) \\
\hline
\end{array}
\]
By the axioms in the above subsections, we can get that $ Lucy $ is a teacher and a staff of the university $ u $, i.e.,  $ \textit{play}(\textit{Lucy}, \textit{teacher}, u, t) $ and $ \textit{play}(\textit{Lucy}, \textit{staff}, u, t) $ hold, as well as other conclusions, such as both $ u $ and \textit{University} having the roles \textit{student}, \textit{teacher} and \textit{staff}.  
\label{Exm:PlayRole}
\end{myExm}

As illustrated in Example \ref{Exm:PlayRole}, in our model, both playing multiple roles simultaneously and playing a same role several times are supported. Moreover, players and social roles have the same identifiers. Finally, we emphasize that currently, we just devote ourselves to capture the general and main relations among contexts, roles and players. Other relationships, such as some social roles having only one player in a same context and same time interval, can be easily modeled by the predicts provided in this section.

\section{Events, Norms, Goals and Desires of Social Contexts, Roles and Players}       
 
Besides the interrelationships, social contexts, roles, social relations and players have their intrinsic content. Take the role \textit{teacher} as an example. Besides the contexts it linked into, we can also associate with the things such as \textit{teaching} and \textit{hoping his students study hard}. Therefor, in this section, we provide a four level structure (Fig. \ref{Fig:FiveS}) to depict the intrinsic information of social contexts, roles, social relations and players mainly from the perspectives of \textit{event}, \textit{norm}, \textit{goal} and \textit{desire}, and illustrate the way of harvesting the intrinsic information of social contexts, roles and relations from the performance of players in concrete contexts. Other elements, such as \textit{belief}, can be analyzed in a similar way.    

\begin{figure}[!htbp] 
	\centering
	\small
	\includegraphics[scale=0.9]{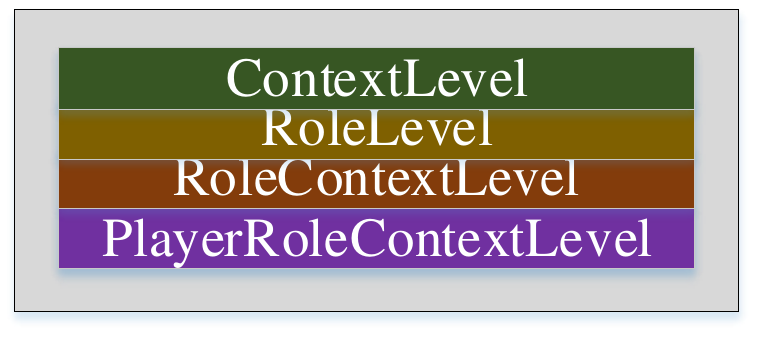}
	\caption{Four level structure of inherent information of social contexts, roles and players.}
	\label{Fig:FiveS}	
\end{figure}

The primitive predicts used and their means as well as the restrictions of the arguments of relational predicts are illustrated in Fig. \ref{Fig:Predict2}. Next, we explain each level gradually. In the following, we use $ \mathcal{S} $ to denote $\{E,N,D,G \}$.

\begin{figure}[!htbp] 
\centering
\small
\subfigure{
\begin{tabular}{|l|}   
	\hline
	\!$ \bullet\ \textit{Event}(e) $, $ \textit{Norm}(n) $:\  {\color{mygray}\textit{$ e $ is an event, $ n $ is a norm}}; \\
	\!$ \bullet\ \textit{Goal}(g) $, $ \textit{Desire}(d) $:\  {\color{mygray}\textit{$ g $ is a goal, $ d $ is a desire.}}; \\
	\!$ \bullet\ \textit{hasE}^{c}(c,e) $:\ {\color{mygray}\textit{context $ c $ has event $ e $.}}\\
	\!$ \bullet\ \textit{hasE}^{r}(r,e) $:\ {\color{mygray}\textit{role $ r $ has event $ e $.}}\\
	\!$ \bullet\ \textit{hasE}^{cor}(r_{1},r_{2},e) $:\ {\color{mygray}\textit{role $ r_{1} $ has event $ e $ to role $ r_{2} $.}}\\
	\!$ \bullet\ \textit{hasE}^{rc}(r,c,e) $:\ {\color{mygray}\textit{role $ r $ has event $ e $ in context $ c $.}}\\
	\!$ \bullet\ \textit{hasE}^{rc}(r_{1},r_{2},c,e) $:\ {\color{mygray}\textit{role $ r_{1} $ has event $ e $ on role $ r_{2} $ in context $ c $.}}\!\\
	\!$ \bullet\ \textit{hasE}_{p}^{rc}(p,r,c,e,t) $:\  {\color{mygray}\textit{player \textit{p} as role \textit{r} in context \textit{c} has event \textit{e}}}\\
	\ \ \! {\color{mygray} \textit{during t}}.\\
	\!$ \bullet\ \textit{hasE}_{cop}^{rc}(p_{1},p_{2},r_{1},r_{2},c,e,t) $:\ {\color{mygray}\textit{player $ p_{1} $ as role $ r_{1} $ in context $ c $ }} \!\!\!\!\!\!\!\!\!\!\!\!\!\\
	 \ \ \!{\color{mygray}\textit{ has event \textit{e} for player $ p_{2} $ as role $ r_{2} $ during \textit{t}}}. \\
	\hline
\end{tabular}  
}
\subfigure{  
\begin{tabular}{|crcl|}   
	\hline
	\!(16)&$ \textit{hasE}^{c}(c,e)$\!\!\!\!\!\!&$\to$&\!\!\!\!\!\!$ (\textit{SC}(c)\vee\textit{ISC}(c))\wedge \textit{Event}(e) $\\
	\!(17)&$ \textit{hasE}^{r}(r,e)$\!\!\!\!\!\!&$\to$&\!\!\!\!\!\!$\textit{SR}(r)\wedge \textit{Event}(e) $\\
	\!(18)&$ \textit{hasE}^{cor}(r_{1}, r_{2},e)$\!\!\!\!\!\!&$\to$&\!\!\!\!\!\!$\textit{SR}(r_{1})\wedge\textit{SR}(r_{2})\wedge \textit{Event}(e) $\\
	\!(19)&$ \textit{hasE}^{rc}(r,c,e)$\!\!\!\!\!\!&$\to$&\!\!\!\!\!\!$\textit{hasR}(c,r)\wedge \textit{Event}(e) $\\
	\!(20)& $ \textit{hasE}^{corc}(r_{1},r_{2},c,e)$\!\!\!\!\!\!&$\to$&\!\!\!\!\!\!$\textit{hasCoR}(c,r_{1},r_{2})\wedge \textit{Event}(e) $\ \ \ \!\!\\
	\!(21)&$ \textit{hasE}_{p}^{rc}(p,r,c,e,t)$\!\!\!\!\!\!&$\to$&\!\!\!\!\!\!$ \textit{play}(p,r,c,t)\wedge \textit{Event}(e) $\\
	\!(22)&$ \textit{hasE}_{cop}^{rc}(\vec{x},e,t)$\!\!\!\!\!\!&$\to$&\!\!\!\!\!\!$ \textit{coPlay}(\vec{x},t)\wedge\textit{Event}(e) $\\
	\hline
\end{tabular}  		
}
\subfigure{
\begin{tabular}{|p{8cm}|} 
	\hline
	The means, dimensions and restrictions of the predicts $ \textit{hasX}^{c} $, $ \textit{hasX}^{r} $, $ \textit{hasX}^{rc} $, $ \textit{hasX}^{corc} $, $ \textit{hasX}_{p}^{rc} $, $ \textit{hasX}_{cop}^{rc} $ can be obtained similarly, where $ X\in\{N, D, G \} $.\\
	\hline
\end{tabular}
}
\caption{Predicts used and their means as well as the restrictions of the arguments of multivariate predicates.}
\label{Fig:Predict2}	
\end{figure}

\paragraph{PlayerRoleContextLevel:}
As a social role in a social context, a player will generate behaviors and mental attitudes and obey some norms relating to this role it playing. Take a teacher in an university as an example. During the position, it needs to teaching, and at the same time, it needs to write papers and hopes to be promoted. This level models such common sense of players in social contexts as the following axiom 23: 
\[
\!\!\!\!\!\!\!\!\!
\begin{array}{crcl}
(23)&\textit{play}(\vec{x},t)\!\!\!\!&\to&\!\!\!\!\bigwedge_{X\in\mathcal{S}}(\exists x,t_{1}.\textit{hasX}_{p}^{rc}(\vec{x},x,t_{1})\\
&&&\qquad\qquad\quad\wedge t_{1}\!\!\prec t)\\
(24)&\textit{coPlay}(\vec{x},t)\!\!\!\!&\to&\!\!\!\!\bigwedge_{X\in\mathcal{S}}(\exists x,t_{1}.\textit{hasX}_{coP}^{rc}(\vec{x},x,t_{1})\\
&&&\qquad\qquad\quad\wedge t_{1}\!\!\prec t)\\
\end{array}
\]   
On the other hand, when materializing a social relation in an context with another player, one player will also generate soem behaviors to the another player to interpret the meaning of this social relation. For example, a teacher in an university needs to guide its students, it may also criticize its student. And a mother in a family needs to care for her daughters and sons, and at the same time hope her children in good health. This level models such common sense of instances of social relations in social contexts as the above axiom 24. 



\paragraph{RoleContextLevel:}  
For the role \textit{student} in schools, we can associate with \textit{attending\_class} and \textit{examination}. And for the role \textit{doctor} in hospitals, we can associate with \textit{treating} patients and forbidding bribery from patients and their relations. This means that social roles and relations in contexts have their intrinsic information. The \textit{RoleContextLevel} captures such common sense of social roles and relations in social contexts as the following two axioms:
\[
\!\!\!
\begin{array}{crcl}
(25)&\textit{hasR}(c,r) \!\!\!\!&\to&\!\!\!\!\bigwedge_{\textit{X}\in \mathcal{S} }(\exists x.\textit{hasX}^{rc}(r,c, x) )\\
(26)&\textit{hasCoR}(c,r_{1},r_{2}) \!\!\!\!\!&\to&\!\!\!\!\!\bigwedge_{\textit{X}\in \mathcal{S} }(\exists x.\textit{hasX}^{corc}(r_{1},r_{2}, c, x)) \\ 
\end{array} 
\]

And the intrinsic information of roles and social relations in instance contexts consist of the information exhibited by the players of these roles and social relations in these contexts, i.e., each player of a role (social relation) in a context can only have information linked to this role (social relation) in this context and each information linked to a role (social relation) in a context will be inherited by one of its player:
\[
\!\!\!\!\!\!\!\!\!\!
\begin{array}{ll}
(27)&\bigwedge_{X\in\mathcal{S}}\big(\textit{ISC}(c)\wedge \textit{hasX}^{rc}(r,c,x)\\
&\qquad\quad\ \leftrightarrow \exists p,t.\textit{hasX}_{p}^{rc}(p,r,c,x,t)\big)\\
(28)&\bigwedge_{X\in\mathcal{S}}\big(\textit{ISC}(c)\wedge \textit{hasX}^{corc}(r_{1},r_{2},c,x)\\
&\qquad\quad\ \leftrightarrow \exists p,t.\textit{hasX}_{cop}^{rc}(p_{1},p_{2},r_{1},r_{2},c,x,t)\big)\\
\end{array} 
\]
Note that in axiom 27-28, ``$ \exists p $" and ``$ \exists p_{1},p_{2} $ cannot be replaced by ``$ \forall p$" and ``$ \forall p_{1}, p_{2} $, since not all information linked into roles and social relations will be exhibited by all of their players. For example, both \textit{attending\_class} and \textit{skipping\_class} are events linked into the role \textit{student} in a school, but not each student of this school will carry out both of these two events. By axiom 27-28, the intrinsic information of social roles and social relations in instance contexts can be gathered from the information exhibited by the players of the roles and materialization of the social relations in the contexts. For example, the events linke into the role \textit{student} in a concrete \textit{school} $ s $ can be gathered from the events exhibited by the students of \textit{s}, i.e.:
\[
\!\!\!\!\!\!\!\!\!
\begin{array}{ll}
(e5) & \textit{hasE}^{rc}(\textit{student}, \textit{s}, e) \exists p, t.\textit{hasE}_{p}^{rc}(p, \textit{student}, c, e, t)\\
\end{array}
\] 

Similarly, the intrinsic information linked into the roles and social relations in abstract contexts consist of the information linked into the roles and social relations in their instance contexts, i.e., a role (social relation) in an instance context can only have information linked to this role (social relation) in the corresponding abstract contexts and each piece of information linked into this role (social relation) in an abstract context will be inherited by this role (social relation) in an instance context : 
\[
\!\!\!\!\!\!\!\!\!\!
\begin{array}{ll}
(29)&\bigwedge_{X\in\mathcal{S}}\big(\textit{SC}(C)\wedge\textit{hasX}^{rc}(r,C,x)\\
&\qquad\quad\leftrightarrow \exists c.\textit{insC}(c,C)\wedge\textit{hasX}^{rc}(r,c,x)\big)\\ 
(30)&\bigwedge_{X\in\mathcal{S}}\big(\textit{SC}(C)\wedge\textit{hasX}^{corc}(r_{1},r_{2}, C,x)\\
&\qquad\quad\leftrightarrow \exists c.\textit{insC}(c,C)\wedge\textit{hasX}^{corc}(r_{1},r_{2},c,x)\big)\\ 
\end{array}
\] 
Analogous to axiom 27-28, in axiom 29-30, ``$ \exists c $" cannot be replaced by ``$ \forall c $". For example \textit{skipping\_class} is an event linked into \textit{student} of \textit{school}, but some schools may not happen the event of \textit{skipping\_class}. According to axiom 29-30, the intrinsic information linked into roles and social relations in abstract contexts can be harvested from the information linked into these roles and social relations in their instance contexts. For example, the events linke into the role \textit{student} in \textit{school} can be obtained from the events linked into the role \textit{student} in concrete schools, i.e.:
\[
\!\!\!\!\!\!\!\!\!\!\!
\begin{array}{ll}
(e6) & \textit{hasE}^{rc}(\textit{student}, \textit{school}, e)\\
     &\leftrightarrow \exists c.\textit{insC}(c, \textit{student})\wedge \textit{hasE}^{rc}(\textit{student}, c, e)\\
\end{array}
\] 

Combining axiom 27-30, the intrinsic information linked into roles and social relations in abstract contexts, i.e., context types, can finally be harvested from the performance of the players of these roles and social relations in instance contexts. 

\paragraph{RoleLevel:}   
From the role and social relation perspective, for the event \textit{teaching}, we can associate with the role \textit{teacher} rather than \textit{doctor}, and for the event \textit{operation}, we can associate with \textit{doctor\_patient} rather than \textit{teacher\_student}. Therefor, social roles and social relations have their intrinsic information. And the \textit{RoleLevel} captures such common sense of roles and social relations as the following two axioms: 
\[
\!\!\!\!\!\!\!\!\!
\begin{array}{crcl}
(31) &\textit{SR}(r)\!\!\!\!&\to&\!\!\!\! \bigwedge_{\textit{X}\in\mathcal{S}}(\exists x.\textit{hasX}^{r}(r, x))\\
(32) &\textit{SR'}(r_{1},r_{2})\!\!\!\!&\to&\!\!\!\! \bigwedge_{\textit{X}\in\mathcal{S}}\big(\exists x.\textit{hasX}^{cor}(r_{1},r_{2}, x))\\ 
\end{array}
\] 
where $ \textit{SR'}(r_{1}, r_{2}) $ denotes $ (r_{1}, r_{2}) $ is a social relation.

And the intrinsic information linked into social roles and social relations consist of the intrinsic information linked into these roles and relations in social contexts, i.e., social roles (relations) in contexts can only have information linked into these roles (relations) and each piece of information linked into a role (relation) will be exhibited by this role (relation) in a social context:
\[
\!\!\!\!\!\!\!\!\!\!
\begin{array}{ll}
(33) & \bigwedge_{X\in\mathcal{S}}\big(\textit{hasX}^{r}(r,x)\leftrightarrow\exists C.\textit{hasX}^{rc}(r,C,x)\big)\\
(34) & \bigwedge_{X\in\mathcal{S}}\big(\textit{hasX}^{cor}(r_{1},r_{2},  x)\\
&\qquad\quad\leftrightarrow\exists C.\textit{hasX}^{corc}(r_{1},r_{2},C,x)\big)\\
\end{array}
\] 

Based on axiom 33-34, the intrinsic information of social roles and relations can be excavated from the intrinsic information linked into these roles and relations in social contexts. For example, the events of the role \textit{doctor} can be obtained from the events of \textit{doctor} in social contexts, i.e.,:
\[
\!\!\!\!\!\!\!\!\!\!
\begin{array}{ll}
(e7)&\textit{hasE}^{r}(\textit{student}, e)\leftrightarrow \exists C.\textit{hasE}(\textit{student}, C, e)\\
\end{array}
\]

%

\paragraph{ContextLevel:} 
From the social context perspective, for the event \textit{teaching}, we can associate with \textit{school} rather \textit{hospital}, and for the event \textit{operation}, we can associate with \textit{hospital} rather than \textit{school}. This means that social contexts have their intrinsic information. And the \textit{ContextLevel} captures the intrinsic information of social contexts as the following axiom:
\[
\!\!\!\!\!\!\!\!\!
\begin{array}{ll}
(35)&\textit{SC}(c)\vee \textit{ISC}(c)\!\to\! \bigwedge_{X\in\mathcal{S}}(\exists x.\textit{hasX}^{c}(c, x))\\
\end{array}
\] 

And the intrinsic information of a social context consists of the intrinsic information linked into the roles and social relations of this context, i.e., roles and social relations can only have information linked into this context and each piece of information of this context will be inherited by a role or relation of this context: 
\[
\!\!\!\!
\begin{array}{llll}
(36)&\bigwedge_{X\in\mathcal{S}}\big(\textit{hasX}^{c}(c,x) \\
 & \leftrightarrow \exists r.\textit{hasX}^{rc}(r,c,x)\vee\exists r_{1},r_{2}.\textit{hasX}^{corc}(r_{1},r_{2},c,x) \big) \\
\end{array} 
\]

Then we can further obtain that the intrinsic information linked into abstarct contexts consists of the intrinsic information linked into their instances, as well as the inheritance of the intrinsic information from sub-contexts to sup-contexts, i.e., the intrinsic information linked to sub-contexts is also the intrinsic information of their sup-contexts. These conclusions are illustrated in the theorem below.

\begin{myTheo}
For social contexts, axiom $ c7 $-$ c8 $ hold: 
\[
\!\!\!\!\!\!\!\!\!\!
\begin{array}{ll}
(c7) &\bigwedge_{X\in\mathcal{S}}\big(\textit{SC}(C)\wedge \textit{hasX}^{c}(C, x)\\
&\qquad\quad\leftrightarrow\exists c. \textit{insC}(c, C)\!\wedge\!\textit{hasX}^{c}(c,x)\big)\\
(c8)&\bigwedge_{X\in\mathcal{S}}\big(\textit{hasX}^{c}(C_{1},x)\wedge\textit{isAC}(C_{1},C_{2})\\
    &\qquad\quad\to \textit{hasX}^{c}(C_{2},x)\big)\\
\end{array}
\]
\end{myTheo} 
\begin{proof}
(Sketch) By axiom 36 and 29-30, axiom $ c8 $ holds. Then by axiom $c7$ and $ d1 $, axiom $ c8 $ holds.
\end{proof} 

Note that the information inheritance from sub-contexts to sup-contexts accords with the common sense of our real world. For example, not only the event \textit{morning\_reading} usually happened in primary schools but also the event \textit{writing\_paper} usually occurred in universities are all events in the school category.

\vspace{0.5em}
\paragraph{Example.} Next, we use an example to illustrate the structure provided in this section.
\begin{myExm}
The description ``\textit{In hospital $ h $, doctor Lucy treated his patient Bob during time interval t. And, Lucy is forbidden to accept bribe during her posting $ t' $.}" can be captured by our model as the following statements:
\[
\!\!\!\!\!\!
\begin{array}{|ll|}
\hline
(e8) &\textit{hasE}_{cop}^{rc}(\textit{Lucy}, \textit{doctor}, \textit{Bob}, \textit{patient}, h, \textit{treating}, t)\ \ \\
(e9) &\textit{hasN}_{p}^{rc}(\textit{Lucy}, \textit{doctor}, h, \textit{no\_bride}, t')\\
(e10) & \textit{insC}(h, \textit{hospital})\\
\hline
\end{array} 
\] 
According to the axioms in this section, we can obtain that the role \textit{doctor} in $ u $ has the event \textit{treating} and norm \textit{no\_bride}, i.e., $ \textit{hasN}^{rc}(\textit{doctor}, h, \textit{no\_bride}) $ and $ \textit{hasE}^{rc}(\textit{doctor}, h, \textit{treating}) $ hold, as well as other conclusions such as both the role \textit{doctor} and the context $ \textit{hospital} $ has the event \textit{treating} and norm \textit{no\_bride}.
\end{myExm} 

\section{Social Context Embedding} 
By deeply analyzing some real-world social contexts, we observe that a concrete context, such as a \textit{school}, actually consists of some smaller and simper contexts, such as \textit{classroom context} and \textit{examination context}, and social roles, such as \textit{teacher}, by entering into some smaller contexts, such as \textit{classroom}, to interpret the meaning of the roles in the original context.
In this section, we formalize this phenomenon and introduce the concept of social context embedding as well as illustrate the way of harvesting the intrinsic information of roles and contexts from the embedded smaller and simpler contexts.

\subsection{Embedding of instance social contexts}   
In this subsection, we illustrate instance social context embedding. For preciseness and comprehensiveness, two types of instance social context embedding, i.e., \textit{embedding} and \textit{role embedding}, are defined. Before that we first depict the observation that each instance context has location and life-cycle information. For this, the predicts $ \textit{LC}(l) $, $ \textit{hasL}(c,l) $ and $ \textit{hasT}(c,t) $ are introduced, respectively denoting that $ l $ is a location, context $ c $ has location $ l $,  and context $ c $ has life-cycle $ t $. The observation of instance contexts having space-time information is captured by the axiom below.
\[
\!\!\!\!\!\!\!\!\!
\begin{array}{lrcl}
(37) & \textit{hasL}(c,l)\!\!\!\!&\to&\!\!\!\!\textit{IC}(c)\wedge \textit{LC}(l)\\
(38) & \textit{hasT}(c,t)\!\!\!\!&\to&\!\!\!\!\textit{IC}(c)\wedge \textit{TI}(t)\\
(39) & \textit{IC}(c)\!\!\!\!&\to&\!\!\!\!\exists l.\textit{hasL}(c,l)\wedge \exists t.\textit{hasT}(c,t)\\
\end{array}
\] 
Here, we mention that for simplicity, we omit the exact definition of locations and their containment relationship.  

As depicted in the following definition, \textit{embedding} depicts the contexts happened in an instance context. Here the \textit{happened-in} relation is captured by the containment of their space-time region.

\begin{myDef}
We use  $ \textit{EIC}(c_{1},c_{2}) $ to denote instance context $ c_{1} $ is embedded into instance context $ c_{2} $, and it is formally defined as:
\[
\!\!\!
\begin{array}{clll}
(d3) & \textit{EIC}(c_{1},c_{2})\!\!\!\!& \stackrel{\textit{def}}{=} &\!\!\!\! \exists l_{1},l_{2},t_{1},t_{2}.(\textit{hasL}(c_{1},l_{1})\wedge \textit{hasL}(c_{2},l_{2})\\
&&&\qquad\qquad\quad\wedge l_{1}\!\subset l_{2}\wedge \textit{hasT}(c_{1},t_{1})\\
&&&\qquad\qquad\quad\wedge\textit{hasT}(c_{2},t_{2})\wedge t_{1}\!\prec t_{2})\\
\end{array}
\]
\label{Def:ebdIC}
\end{myDef} 

By Definition \ref{Def:ebdIC}, we can get that each classroom context $ \textit{cr} $ happened in a school $ s $ is an embedded instance context of this school, i.e., $ \textit{EIC}(cr,s) $ holds.  

On the other hand, from the perspective of social roles, \textit{role embedding} depicts the phenomenon of a role of a context entering into an embedded context to play another role. For example, a teacher of a school will enter into a classroom context to play the role of lecturer. In our theory, we do not generate role instances. Thus, the \textit{role playing role} is captured by sharing the same players.  

\begin{myDef}
We use  $ \textit{REIC}(c_{1},r_{1},c_{2},r_{2}) $ to denote role $ r_{2} $ of $ c_{2} $ is entered into the embeded context $ c_{1} $ to play role $ r_{1} $:
\[
\!\!\!\!\!\!\!\!\!
\begin{array}{lrcl}
(d4)&\textit{REIC}(c_{1},r_{1}, c_{2}, r_{2})\!\!\!\!&\stackrel{\textit{def}}{=}&\!\!\!\!\textit{EIC}(c_{1},c_{2})\wedge \textit{hasR}(c_{1},r_{1}) \\
&&&\!\!\!\!\wedge \textit{hasR}(c_{2},r_{2})\\
&&&\!\!\!\!\wedge \exists p,t.(\textit{play}(p,r_{1},c_{1},t)\\
&&&\qquad\ \ \wedge\textit{play}(p, r_{2},c_{2},t))\\  
\end{array}
\]
\label{Def:rebdIC}
\end{myDef} 

Note that in Definition \ref{Def:rebdIC}, $ r_{1} $ and $ r_{2} $ can be the same roles. By Definition \ref{Def:ebdIC} and \ref{Def:rebdIC}, we can further obtain the reflexivity and transitivity of \textit{embedding} as well as the reflexivity of \textit{role embedding}.

\begin{myTheo}
For instance contexts, axiom $ c10$-$c12 $ hold:  
\[
\!\!\!\!\!\!\!\!\!
\begin{array}{lrcl}
(c10)& \textit{IC}(c)\!\!\!\!&\to &\!\!\!\! \textit{EIC}(c,c)\\
(c11) & \textit{EIC}(c_{1},c_{2})\wedge\textit{EIC}(c_{2},c_{3})\!\!\!\!&\to&\!\!\!\! \textit{EIC}(c_{1},c_{3})\\
(c12) & \textit{ISC}(c)\wedge \textit{hasR}(c,r)\!\!\!\!&\to&\!\!\!\! \textit{REIC}(c,r,c,r)\\
\end{array}
\]
\label{The:ebdIC}	
\end{myTheo}
\begin{proof}
By axioms $ d3\!-\!d4 $, these three axioms hold trivially.
\end{proof}

Next, we illustrate the information inheritance between instance contexts and their embedded smaller and simpler contexts. 

From the player perspective, we found that as a role of a social context, a player explains the meaning and mission of this role by  entering into some embedded contexts. Such foundation is formalized in the axiom below from the perspective of events, norms, goals and desires: 
\[
\!\!\!
\begin{array}{ll}
(40) & \bigwedge_{X\in\mathcal{S}}\big(\textit{hasX}_{p}^{rc}(p, r, c, x, t)\\
&\qquad\quad\leftrightarrow \exists r',c',t'.(\textit{EIC}(c',c)\\
& \qquad\qquad\qquad\quad\ \ \wedge \textit{hasX}_{p}^{rc}(p,r',c',x,t')\wedge t'\prec t)\big)\\
\end{array}
\]
For example, as a student of a school, the events of \textit{attending\_class} and \textit{skipping\_class} are exhibited in a classroom context, and the goal of \textit{obtaining\_high\_score} is exhibited in an examination context. Similarly, in a context, the behaviors and mental attitudes of one player to another player which materializes or co-plays a social relation with it, is also embodied in some embedded contexts:
\[
\!\!\!\!
\begin{array}{ll} 
(41) & \bigwedge_{X\in\mathcal{S}}\big(\textit{hasX}_{cop}^{rc}(p_{1},p_{2}, r_{1},r_{2}, c, x, t)\\
&\qquad\quad\leftrightarrow \exists r_{1}',r_{2}',c',t'.(\textit{EIC}(c',c)\\
& \qquad\qquad\qquad\qquad\ \wedge \textit{hasX}_{cop}^{rc}(p_{1},p_{2},r_{1}',r_{2}',c',x,t')\\
& \qquad\qquad\qquad\qquad\ \wedge t'\!\prec t)\big)\\
\end{array}
\]
For example, in a school, the events of \textit{teaching} and \textit{corporal punishment} as well as the desire of \textit{paying attention to classes} of teachers to students all exhibited in some embedded classroom contexts.

Then, we can further obtain that the intrinsic information, i.e., \textit{events}, \textit{norms}, \textit{goals} and \textit{desires}, linked into roles and social relations in instance contexts can be obtained from the corresponding roles and social relations in some embedded contexts. We formalize this conclusion in the following theorem.

\begin{myTheo}
For instance contexts, axiom $ c12 $-$ c14 $ hold: 
\[
\!\!\!\!\!\!\!\!\!
\begin{array}{ll}
(c13) & \bigwedge_{X\in\mathcal{S}}\big(\textit{IC}(c)\wedge \textit{hasX}^{rc}(r,c,x) \\
&\qquad\quad\leftrightarrow \exists c',r'.\textit{REIC}(c',r',c, r)\\
&\qquad\qquad\qquad\quad\ \wedge\textit{hasX}^{rc}(r',c',x)\big)\\
(c14) & \bigwedge_{X\in\mathcal{S}}\big(\textit{IC}(c)\wedge \textit{hasX}^{corc}(r_{1}, r_{2},c,x)\\
&\qquad\quad\leftrightarrow\exists r_{1}',r_{2}',c'.\textit{REIC}(c',r_{1}',c, r_{1}) \\
&\qquad\qquad\qquad\qquad\ \ \wedge\textit{REIC}(c',r_{2}',c, r_{2})\\
&\qquad\qquad\qquad\qquad\ \ \wedge\textit{hasX}^{corc}(r_{1}',r_{2}',c',x)\big)\\
\end{array}
\]
\end{myTheo}
\begin{proof}
(Sketch) According to axiom 40, 27 and $ d4 $, axiom $ c13 $ holds. And According to axiom 40, 28 and $ d4 $, axiom $ c14 $ holds.
\end{proof}

Moreover, by axiom $ c13-c14 $ and axiom 33-34 as well as axiom 36, we can also obtain that the intrinsic information linked into social roles. social relations and instance contexts can finally be obtained from the embedded simpler and smaller instance contexts. 

\subsection{Embedding of abstract social contexts}  
Based on instance context embedding, in this subsection, we formalize abstract context embedding as well as illustrate the way of harvesting the information linked into roles and social relations in abstract contexts from the embedded simper and smaller contexts. Again for preciseness and comprehensive, two types of abstract social context embedding, i.e., \textit{embedding} and \textit{role embedding}, are defined.

As illustrated in the definition below, \textit{embedding} depicts the types of contexts that may be embedded into the instances of abstract contexts.

\begin{myDef}
We use $ \textit{ESC}(C_{1}, C_{2}) $ to denote abstract context $ C_{1} $ is embedded into abstract context $ C_{2} $, and it is defined as: 
\[
\!\!\!
\begin{array}{llll}
(d5) & \textit{ESC}(C_{1},C_{2})\!\!\!\!&\stackrel{\textit{def}}{=}&\!\!\!\!\textit{SC}(C_{1})\wedge\textit{SC}(C_{2})\\
&&&\!\!\!\!\!\wedge \exists c_{1},c_{2}.(\textit{insC}(c_{1},C_{1})\\
&&&\quad\ \ \wedge\textit{insC}(c_{2},C_{2})\wedge \textit{EIC}(c_{1},c_{2})) \\
\end{array}
\]
\label{Def:pebdC}
\end{myDef}

By Definition \ref{Def:pebdC}, we can get that the \textit{classroom} context is an embedded context of the school context, i.e., $ \textit{ESC}(\textit{classroom}, \textit{school}) $ holds. Moreover, by Definition \ref{Def:pebdC}, we can obtain the reflexivity of abstract context embedding as well as the extension of embedding on the sub-context relation, i.e., if $ C_{1} $ is an embedded context of $ C_{2} $ then for each sup-context $ C_{3} $ of $ C_{2} $, $ C_{1} $ is also an embedded context of $ C_{3} $, and for each sup-context $ C_{3} $ of $ C_{1} $, $ C_{3} $ is also an embedded context of $ C_{2} $.  

\begin{myTheo}
For abstract contexts, axiom $ c15 $-$ c17 $ hold:
\[
\!\!\!\!\!\!\!\!
\begin{array}{crcl}
(c15) & \textit{SC}(C)\!\!\!\!&\to &\!\!\!\!\textit{ESC}(C,C)\\
(c16) & \textit{ESC}(C_{1}, C_{2})\wedge \textit{isAC}(C_{2},C_{3})\!\!\!\!&\to &\!\!\!\!\textit{ESC}(C_{1},C_{3})\\ 
(c17) & \textit{ESC}(C_{1}, C_{2})\wedge \textit{isAC}(C_{1},C_{3})\!\!\!\!&\to &\!\!\!\!\textit{ESC}(C_{3},C_{2})\\  
\end{array}
\]		
\end{myTheo}
\begin{proof}
	By axiom $ d5 $ and $ d1 $, these three axioms hold.
\end{proof} 

Note that this extension also accords with the common sense of our real world. For example, \textit{classroom} context is an embedded context of the \textit{primary\_school} context, and \textit{primary\_school} context is a sub-context of \textit{school} context, then we naturally consider that \textit{classroom} context is also an embedded context of \textit{school} context. And \textit{soccer game} is an embedded context of the context \textit{school}, and \textit{soccer game} is a sub-context of the context \textit{game}, then we naturally consider that \textit{game} is also an embedded context of \textit{school}.

On the other hand, from the perspective of social roles, \textit{role embedding} depicts the types and roles of embedded contexts a role of an abstract context may enter into and play.

\begin{myDef}
We use $\textit{RESC}(C_{1},r_{1}, C_{2}, r_{2}) $ to denote that the role $ r_{2} $ of $ C_{2} $ may enter into the embedded context $ C_{2} $ to play $ r_{2} $, and it is defined as: 
\[
\!\!\!
\begin{array}{llll}
(d6) & \textit{RESC}(C_{1},r_{1},C_{2}, r_{2})\!\!\!\! & \stackrel{\textit{def}}{=} &\!\!\!\!\textit{ESC}(C_{1},C_{2})\\
&&&\!\!\!\!\wedge\exists c_{1}, c_{2}.(\textit{insC}(c_{1},C_{1})\\
&&&\ \quad\wedge\textit{insC}(c_{2},C_{2})\\
&&&\ \quad\wedge\textit{REIC}(c_{1},r_{1},c_{2},r_{2})\!)
\end{array}
\] 
\label{Def:rpebdC}	 	
\end{myDef} 

By Definition \ref{Def:rpebdC}, the knowledge of ``teachers of schools may enter into classrooms to play the role of lecturers" can be formalized as $ \textit{RESC}(\textit{classroom}, \textit{lecture}, \textit{school}, \textit{teacher}) $.  

As shown in the theorem below, \ref{Def:rpebdC}, we can further obtain the reflexivity of \textit{role embedding} as well as the extension of \textit{role embedding} on the \textit{sub-role} relation. 

\begin{myTheo}
For abstract contexts, axiom $ c18 $-$ c20 $ hold:  
\[
\!\!\!\!\!\!\!\!
\begin{array}{ll}
(c18) & \textit{SC}(C)\wedge \textit{hasR}(C, r)\to \textit{RESC}(C,r,C,r)\\
(c19) & \textit{RESC}(C_{1}, r_{1},C_{2}, r_{2})\wedge \textit{isAR}(r_{2},r_{3})\\
&\to \textit{RESC}(C_{1}, r_{1},C_{2}, r_{3})\\
(c20) & \textit{RESC}(C_{1}, r_{1},C_{2}, r_{2})\wedge \textit{isAR}(r_{1},r_{3})\\
&\to \textit{RESC}(C_{1}, r_{3},C_{2}, r_{2})\\
\end{array}
\]	
\end{myTheo} 
\begin{proof}
	By axiom $ d6 $ and $ d2 $, axiom $ c18 $-$ c20 $ hold. 
\end{proof}

Finally, based on the information inheritance between instance contexts and their embedded contexts, we can finally obtained that the intrinsic information linked into roles and social relations in abstract contexts can be harvested from the corresponding roles and social relations in their embedded abstract contexts.

\begin{myTheo}
For abstract contexts, axioms $ c21 $-$ c22 $ hold:
\[
\!\!\!\!\!\!\!\!\!
\begin{array}{ll}
(c21) & \bigwedge_{X\in\mathcal{S}}\big(\textit{SC}(C)\wedge \textit{hasX}^{rc}(r,C,x) \\
&\qquad\quad\leftrightarrow \exists c',r'.\textit{RESC}(C',r',C, r)\\
&\qquad\qquad\qquad\quad\wedge \textit{hasX}^{rc}(r',C',x)\big)\\
(c22) & \bigwedge_{X\in\mathcal{S}}\big(\textit{SC}(C)\wedge \textit{hasX}^{corc}(r_{1}, r_{2},C,x)\\
&\qquad\quad\leftrightarrow \exists r_{1}',r_{2}',C'.\textit{RESC}(C',r_{1}',C, r_{1}) \\
&\qquad\qquad\qquad\qquad\quad\wedge \textit{RESC}(C',r_{2}',C, r_{2})\\
&\qquad\qquad\qquad\qquad\quad\wedge \textit{hasX}^{corc}(r_{1}',r_{2}',C',x)\big)\\
\end{array}
\]
\end{myTheo}
\begin{proof}
(Sketch) By axiom $ c13 $-$ c14 $, $ d5 $-$ d6 $ and $ 29 $-$ 30 $, these three axioms hold.
\end{proof} 

Moreover, by axiom $ c21 $-$ c22 $ and axiom 33-34 and 36, we can further obtain that the intrinsic information linked into abstract contexts can finally be obtained from the embedded simple and smaller abstract contexts. 

\section{Related Work}

In this section, we introduce the works on describing roles, events, norms, goals and desires. 

\subsection{Describing of roles}
For multi-agent systems, Esteva et al. (2001) describes social roles in terms of expectations, standardized patterns of behaviour, Pacheco and Carmo (2003) in terms of rights, permissions and obligations, Dastani et al. (2004) in terms of goals and planning rules. Boella and Van Der Torre (2007) addresses the problem of defining social roles in multi-agent systems and three ontological properties of roles, i.e., \textit{definitional dependence}, \textit{powers} and \textit{roles playing roles}, are identified and formalized. Colman and Han (2007) focus on building software systems on the basis of social contexts which are intentionally designed and structured in the systems. In these models, roles are treated as agents.

For object-oriented and conceptual modeling, Wieringa (1995) discusses using dynamic classes and role classes to model object migration based on order-sorted dynamic logic and process algebra. In Wieringa (1995), roles are modeled as classes and role instances and their players have different identities. 
Kristensen (1996) defines a graphical notation to support static and dynamic description of roles.  
Steimann (2000) presents a basic definition of roles and demonstrates how it naturally accounts for many modeling issues
In Steimann (2000), roles are modeled as types and role instances and their players have different identities. 
Herrmann (2007) describes the properties of roles in a modern approach which introduces roles in object oriented programming language.

For role classfication, Loebe (2007) classifies roles into three types, i.e., relational roles, processual roles and social roles, based on GFO (Heller et al. 2005). In Loebe (2007), relational roles and processual roles are modeled as unary predicts, however, no formalization of social roles is provided.
Mizoguchi et al. (2007)  presents a framework for organizing role concepts according to their context dependencies, and two kinds of role concepts, i.e., primitive roles and composite roles, are identified. In Mizoguchi et al. (2007), roles are modeled as OWL classes. 
Masolo et al. (2011) focuses on the ways in which roles are specified, examines the formal constraints on their definitions, and proposes definitional schemas motivating different kinds of roles. In Masolo et al. (2011), three types of roles, i.e., participation roles, non-participation roles and historical roles, are identified, and roles are represented as multivariate predicts. Kozaki et al. (2012) classifies roles into original roles and derived roles. 
Mizoguchi et al. (2015)  introduces the idea of a family of occurrent-dependent roles as a means to organise prospective and retrospective derived roles around an original role from which they are derived. In Mizoguchi et al. (2015), roles are modeled as unary predicts.

There also exist some works trying to find problem independent definition and formalization of roles. Based on first-order theory, 
Masolo et al. (2004) depicts the ontological features of social roles, such as anti-rigidity and foundation. In Masolo et al. (2004), roles are modeled as instances, however it does not formalize the relation between contexts and roles. 
Masolo et al. (2005) elaborates the relational feature of roles. 
In Masolo et al. (2005), roles are modeled as unary predicts and roles instances and their players share the same identities.
Boella and van der Torre (2006) identifies the axioms modeling social concepts like organization and roles and the properties distinguishing them from other categories like objects and agents. 
Genovese (2007a and 2007b) aims to provide a flexible formal model for roles which is able to catch the basic primitives behind the different role's accounts in the literature, rather than a definition. In Genovese (2007a and 2007b), roles are modeled as instances.

In this paper, we identify and formalize the interrelations among social contexts, roles and player from the perspective of SCK representation and acquiring, rather than identifying the distinguishing features of roles or contexts or defining roles in special domains. Different to these works, we consider both contexts and roles as first-order citizen and do not generate role instances. Moreover, besides \textit{having} and \textit{playing} social roles, more interrelations, such as \textit{having} and \textit{playing} social relations, are identified and formalized in this paper.




\subsection{Describing of events, norms, goals and desires}
Bell and Huang (1996) introduces and formalizes dynamic goal hierarchies as well as the rational revision of goals and goal hierarchies. Dignum, Kinny and Sonenberg (2002) proposes an extended BDI architecture in which obligations, norms and desires are distinguished from goals and explicitly represented. Fasli (2002) presents a formalization of obligations, social commitments and roles for BDI agents. Broersen et al. (2002) considers goal generation in cognitive agent architectures. Mueller (2014) discusses event calculus in term of commonsense reasoning. 

Different to these works, in this paper, we depict the events, norms, goals and desires of social contexts, roles and players from the perspectives of social common sense representation and reasoning, rather than discusses the definition and revision of events, norms, goals and desires.  

\section{Conclusion, Discussion and Future Work} 
For supporting SCK acquisition, in this paper, we identify and formalize three basic types of SCK. Based on the existing works on roles, we first identify and formalize the interrelationships, such as \textit{having} and \textit{playing} social relations, among social contexts, roles and players from the perspectives of considering both social contexts and roles as first order citizens and not generating role instances. Then we propose a four level structure to identify and formalize the intrinsic information, such as \textit{events} and \textit{desires}, of social contexts, roles and players from different angles and with different granularity. Enlightened by some observations of actual social contexts, we further introduce and formalize the embedding of social contexts, and illustrate the way of harvesting the intrinsic information of contexts and roles from the embedded smaller and simpler contexts is illustrated. 

Clearly, what we propose in this paper is only a first step toward a formal account of SCK representation. Currently, we mainly devote ourselves to identify and formalize what do social contexts and roles have. More complex SCK, like if a patient takes a turn for the worse then his parent will be anxious and teachers belief that if a student studies hard then he will get high scores, is not captured and formalized in this paper. On the other hand, the theory is constructed by deeply analyzing and summarizing the knowledge and phenomenon of social contexts and roles. This means its correctness can not be proved logically. Actually, it cannot be hold as absolute truths, but rather be a reflection of overall tendencies. This is also the essence of common sense.  

Based on the results in this paper, our future work will mainly concentrate on identifying and formalizing much more complex SCK, such as the emotion patterns and behavior tendency of social roles, and acquisition the three basic types of SCK identified and formalized in this paper.    


\section{Acknowledgment} 
This work has been supported by the National Key Research and Development Program of China under grant 2017YFC1700300.

\section{References} 

\smallskip \noindent
Baldoni, M., Borlla, G., and van der Torre, L. 2005. Roles as a coordination construct: Introducing powerJava. In  \textit{Proceedings of MTCoord05}.  

\smallskip \noindent
Bell, J., and Huang, Z. 1996. Dynamic Goal Hierarchies. In \textit{International Workshop on Intelligent Agent Systems}. 

\smallskip \noindent
Bera, P., Burton-Jones, A., and Wand, Y. 2017. Improving the representation of roles in conceptual modeling: theory, method, and evidence. \textit{Requirements Engineering}.

\smallskip \noindent
Boella, G., and van der Torre, L. 2006. A foundational ontology of organizations and roles. In \textit{International Workshop on Declarative Agent Languages and Technologies}. 

\smallskip \noindent
Boella, G., van der Torre, L., and Verhagen, H. 2006. Roles, an Interdisciplinary Perspective. \textit{Applied Ontology} 3(2006): 1--7. 

\smallskip \noindent
Boella, G., and van der Torre, L. 2007. The ontological properties of social roles in multi-agent systems: definitional dependence, powers and roles playing roles. \textit{Artificial Intelligence and Law} 15: 201--221.  

\smallskip \noindent
Broersen, J., Dastani, M., Hulstijn, J., and van der Torre. 2002. Goal Generation in the BOID Architecture. \textit{Cognitive Science Quarterly Journal} 2(3–4): 428--447.    

\smallskip \noindent
Colman, A.W., and Han, J. 2007. Roles, players and adaptable organizations. \textit{Applied Ontology} 2(2): 105--126.

\smallskip \noindent
Davis, E. 2017. Logical Formalizations of Commonsense Reasoning: A Survey. \textit{Journal of Artificial Intelligence Research} 59 (2017): 651--723.

\smallskip \noindent
Davis, E., and Marcus, G. 2015. Commonsense Reasoning and Commonsense Knowledge in Artificial Intelligence. \textit{Communications of The ACM} 58(9): 92--103. 

\smallskip \noindent
Dastani, M., van Riemsdijk, B., Hulstijn, J., Dignum, F., and Meyer, J.J. 2004. Enacting and Deacting Roles in Agent Programming. In \textit{Proceedings of the Agent Oriented Software Engineering Workshop}, 189--204. 

\smallskip \noindent 
Dignum, F., Kinny, D., and Sonenberg, L. 2002. From Desires, Obligations and Norms to Goals. \textit{Cognitive Science Quarterly}, 2(3-4): 407--430. 

\smallskip \noindent
Esteva, M., Padget, J., and Sierra, C. 2001. Formalizing a Language for Institutions and Norms. In \textit{International Workshop on Agent Theories}.  

\smallskip \noindent
Fan, J., Barker, K., Porter B., and Clark, P. 2001. Representing roles and purpose.  In \textit{Proceedings of the 1st International Conference on Knowledge Capture}, 38--43.   

\smallskip \noindent
Fasli, M. 2002. On Commitments, Roles, and Obligations. In \textit{International Workshop of Central and Eastern Europe on Multi-Agent Systems}. 

\smallskip \noindent
Genovese, V. 2007a. A Meta-model for Roles: Introducing Sessions. In: \textit{Proceedings of the 2nd Workshop on Roles and Relationships in Object Oriented Programming, Multiagent Systems, and Ontologies} 27-38.

\smallskip \noindent
Genovese, V. 2007b. Towards a general framework for modelling roles. \textit{Normative Multi-agent Systems 2007}


\smallskip \noindent
Heller, B., Herre, H., Burek, P., Loebe, F., and Michalek, H. 2005.  General Formal Ontology (GFO): A Foundational Ontology Integrating Objects and Processes (Version 1.0 D1). Onto-Med Report Nr. 8, Research Group Ontologies in Medicine (Onto-Med), Leipzig University, Germany.  

\smallskip \noindent
Herrmmann, S. 2007. A precise model for contextual roles: The programming language ObjectTeams/Java. \textit{Applied Ontology} 2(2): 181--207.



\smallskip \noindent
Kristensen B.B. 1996. Object-Oriented Modeling with Roles. In: \textit{Murphy J., Stone B. (eds) OOIS’ 95. Springer, London}.

\smallskip \noindent
Loebe, F. 2007. Abstract vs. social roles - towards a general theoretical account of roles. \textit{Applied Ontology} 2(2): 127--158.

\smallskip \noindent
Masolo, C., Guizzardi, G., Vieu, L., Bottazzi, E., and Ferrario, R. 2005. Relational roles and qua-individuals. In \textit{AAAI 2005 Fall Symposium on Roles, an interdisciplinary perspective (Roles' 05)}.   

\smallskip \noindent
Masolo, C., Vieu, L., Bottazzi, E., Catenacci, C., Ferrario, R., Gangemi, A., and Guarino, N. 2004. Social Roles and Their Descriptions. In \textit{Proceedings of the Ninth International Conference on the Principles of Knowledge Representation and Reasoning} 267--277.  

\smallskip \noindent
Masolo, C., Vieu, L., Kitamura, Y., Kozaki, K., and Mizoguchi, R. 2011. The Counting Problem in the Light of Role Kinds. \textit{Logical Formalizations of Commonsense Reasoning--Papers from the AAAI 2011 Spring Symposium.} 

\smallskip \noindent
Mizoguchi, R., Galton, A., Kitamura, Y., and Kozaki, K. 2015. Families of roles: A new theory of occurrent-dependent roles. \textit{Applied Ontology} 10(2015): 367--399.  

\smallskip \noindent
Mizoguchi, R., Sunagawa, E., Kozaki, K., and Kitamura, Y. 2007. A model of roles in ontology development tool: Hozo. \textit{Applied Ontology} 2(2): 159--179.  

\smallskip \noindent
Mueller, E.T. 2014. Commonsense Reasoning An Event Calculus Based Approach. \textit{Morgan Kaufmann, San Francisco, CA}.



\smallskip \noindent
Pacheco, O., and Carmo, J. 2003. A Role Based Model for Normative Specification of Organized Collective Agency and Agents Interaction. \textit{Autonomous Agents and Multiagent Systems}, 6(2): 145--184.  



\smallskip \noindent
Steimann, F. 2007. The role data model revisited. \textit{Applied Ontology} 2(2): 89--103.

\smallskip \noindent
Steimann, F. 2000. On the representation of roles in object-oriented and conceptual modelling. \textit{Data and Knowledge Engineering} 35: 83--848. 

\smallskip \noindent
Sunagawa, E., Kozaki, K., Kitamura, Y., and Mizoguchi, R. 2006. Role organization model in Hozo. In \textit{Proceedings of the 15th International Conference Knowledge Engineering and Knowledge Management}.

\smallskip \noindent
Tamai, T., Ubayashi, N., and Ichiyama, R. 2005. An adaptive object model with dynamic role binding. In \textit{Proceedings International Conference on Software Engineering} 166--175.



\smallskip \noindent
Wieringa, R., de Jonge, W., and Spruit, P. 1995. Using Dynamic Classes and Role Classes to Model Object Migration. \textit{Theory and Practice of Object Systems} 1(1):61-83.




\end{document}